\documentclass{article}
\usepackage{geometry}

    \newgeometry{vmargin={17mm}, hmargin={12mm,17mm}}   

\usepackage[numbers]{natbib}

\usepackage[utf8]{inputenc} 
\usepackage[T1]{fontenc}    
\usepackage{hyperref}       
\usepackage{url}            
\usepackage{booktabs}       
\usepackage{amsfonts}       
\usepackage{nicefrac}       
\usepackage{microtype}      
\usepackage{xcolor}         

\usepackage{amsmath,amsfonts,bm}

\def\eqref#1{equation~\ref{#1}}

\def\1{\bm{1}}

\DeclareMathAlphabet{\mathsfit}{\encodingdefault}{\sfdefault}{m}{sl}
\SetMathAlphabet{\mathsfit}{bold}{\encodingdefault}{\sfdefault}{bx}{n}

\DeclareMathOperator*{\argmax}{arg\,max}

\usepackage[utf8]{inputenc} 
\usepackage[T1]{fontenc}    
\usepackage{hyperref,wrapfig,optidef}       
\usepackage{url}            
\usepackage{booktabs,amsmath,amsthm,color}       
\usepackage{amsfonts}       
\usepackage{xfrac,paralist}       
\usepackage{microtype}      
\newcommand\numberthis{\addtocounter{equation}{1}\tag{\theequation}}
\newtheorem*{theorem*}{Theorem}
\newtheorem{theorem}{Theorem}
\newtheorem{proposition}{Proposition}
\usepackage{color}
\definecolor{myblue}{rgb}{.8, .8, 1}
\newtheorem{definition}{Definition}
\usepackage{amsmath}

\usepackage{epsfig}
\usepackage{multicol}
\usepackage{vwcol}  
\usepackage{graphicx}
\graphicspath{{figures/}{Figures/}}
\usepackage{amsmath,wrapfig, graphicx,bm}
\usepackage{amssymb}
\usepackage{multirow}
\usepackage{colortbl, array}
\usepackage{hhline}
\usepackage{empheq}
\usepackage{booktabs}

\usepackage{amsmath,amssymb,comment,tabularx}       
\usepackage{url}            
\usepackage{booktabs}       
\usepackage{nicefrac}       
\usepackage{microtype}      
\usepackage[linesnumbered,ruled]{algorithm2e} 
\usepackage{graphicx,caption,color,wrapfig}
\usepackage{accents}

\newcommand{\myrowcolour}{\rowcolor[gray]{0.925}}

\definecolor{rulecolor}{RGB}{70,10,171}
\definecolor{tableheadcolor}{RGB}{120,50,200}
\newcommand{\topline}{ %
        \arrayrulecolor{rulecolor}\specialrule{0.1em}{\abovetopsep}{0pt}}%
\newcommand{\midtopline}{ %
        \arrayrulecolor{rulecolor}\specialrule{\lightrulewidth}{0pt}{0pt}}%
\newcommand{\bottomline}{ %
        \arrayrulecolor{rulecolor} \specialrule{\lightrulewidth}{0pt}{0pt}}%
        
\usepackage{subcaption,wrapfig, xcolor}

\usepackage{tcolorbox}

\newtcolorbox{mybox}[3][]
{
  colframe = #2!15,
  colback  = #2!10,
  coltitle = #2!10!black,  
  title    = {#3},
  boxsep   = 0.25pt,
  left     = 0.5pt,
  right    = 0.5pt,
  top      = 0pt,
  bottom   = 0pt,
  width=\linewidth+13pt,
  #1,
}

\usepackage{picins, wrapfig}
\usepackage{url}

\title{An Online Riemannian PCA for Stochastic Canonical Correlation Analysis}

\author{%
  Zihang~Meng$^+\dagger$, Rudrasis~Chakraborty$^*\dagger$ and Vikas~Singh$^+$ \\
  $+$ University of Wisconsin Madison \\
  $*$ Amazon Lab126\\
  $\dagger$ Equal Contribution \\
  \texttt{zmeng29@wisc.edu, rudrasischa@gmail.com, vsingh@biostat.wisc.edu} \\
}
\date{}
\begin{document}

\maketitle

\begin{abstract}
  We present an efficient stochastic algorithm (RSG+) for canonical correlation analysis (CCA) using a reparametrization of the projection matrices. We show how this reparametrization (into structured matrices), simple in hindsight,  
  directly presents an opportunity to repurpose/adjust 
  mature techniques  
  for numerical optimization on Riemannian manifolds. 
  Our developments nicely 
  complement existing methods for this problem which either require
  $O(d^3)$ time complexity per iteration with 
  $O(\frac{1}{\sqrt{t}})$ convergence rate (where $d$ is the dimensionality)
  or only extract the top $1$ component with $O(\frac{1}{t})$ convergence rate.
  In contrast, our algorithm offers a strict improvement for this classical problem: it achieves $O(d^2k)$ runtime complexity per iteration for
  extracting the top $k$ canonical components with $O(\frac{1}{t})$ convergence rate.
  While the paper primarily focuses 
  on the formulation and technical analysis of its 
  properties, our 
  experiments show that 
  the empirical behavior on 
  common datasets is quite promising. We also explore a potential application in training fair models where the label of protected attribute is 
  missing or otherwise unavailable. 
\end{abstract}

\vspace{-7pt}
\section{Introduction}
\vspace{-3pt}
Canonical correlation analysis (CCA) is a classical
method for evaluating correlations between two sets of variables.
It is commonly used in unsupervised multi-view learning,
where the multiple views of the data may correspond to
image, text, audio and so on
\cite{rupnik2010multi,chaudhuri2009multi,luo2015tensor}.
Classical formulations have also been
extended to leverage advances in
representation learning, for example, \cite{andrew2013deep}
showed how the CCA  can be  interfaced 
with deep neural networks enabling  modern use cases. Many results
over the last few years have used CCA or its variants for problems
including measuring representational similarity in deep neural networks \cite{morcos2018insights},
speech recognition \cite{Couture2019DeepML}, and so on.


The goal in CCA is to find linear combinations within two random variables $\bold{X}$ and $\bold{Y}$
which have maximum correlation with each other. Formally, 
the CCA problem is defined in the following way.
Let $X \in \mathbf{R}^{N\times d_x}$ and $Y \in \mathbf{R}^{N\times d_y}$ be $N$ samples respectively drawn from pair of random variables $\bold{X}$ ($d_x$-variate random variable) and $\bold{Y}$ ($d_y$-variate random variable), with unknown joint probability distribution. The goal is to 
find the projection matrices $U\in \mathbf{R}^{d_x\times k}$ and
$V\in\mathbf{R}^{d_y\times k}$, with $k\le \mathrm{min}\{d_x,d_y\}$, such
that the correlation is maximized:
\begin{align}
\label{cca_opt1}
\max_{U, V} \:\:& F = \text{trace}\left(U^TC_{XY} V \right) \:\: \: \text{ s.t. } & U^T C_X U=I_{k}, V^TC_Y V=I_k 
\end{align}
Here, $C_X = \frac{1}{N}X^TX$ and $C_Y = \frac{1}{N}Y^TY$ are the sample covariance matrices, and $C_{XY}= \frac{1}{N}X^TY$ denotes  the sample cross-covariance.

The objective function in (\ref{cca_opt1})
is the expected cross-correlation in the projected space and the
constraints specify
that different canonical components should be decorrelated. Let us define the whitened covariance $T\coloneqq C_X^{-1/2}C_{XY}C_Y^{-1/2}$ and
$\Phi_k$ (and $\Psi_k$) contains the top-$k$ left (and right) singular vectors of $T$. It is known
\cite{Golub1992TheCC} that the optimum of (\ref{cca_opt1})
is achieved at $U^*=C_X^{-1/2}\Phi_k$, $V^*=C_Y^{-1/2}\Psi_k$. Finally, we compute $U^*$, $V^*$ by applying a $k$-truncated SVD to $\widetilde{T}$.


{\bf Runtime and memory considerations.}
The above procedure is simple but is only feasible when the data matrices are small.
In most modern applications, not only are
the datasets large but also the dimension $d$ (let $d = \max\{d_x, d_y\}$) of each sample can be  high, especially
if representations are being learned using deep models. 
As a result, the computational footprint of the algorithm can be high. 
This has motivated the study of stochastic optimization routines
for solving CCA. 
Observe that in contrast to the typical settings where stochastic  schemes are most effective,
the CCA objective does not decompose over samples in the dataset. 
Many efficient strategies have been proposed in the literature: 
for example, \cite{ge2016efficient,wang2016efficient} present
Empirical Risk Minimization (ERM) models which optimize
the empirical objective.
More recently, \cite{gao2019stochastic,bhatia2018gen,arora2017stochastic} describe
proposals that optimize the population objective. To summarize the approaches succinctly,
if we are satisfied with identifying the top $1$ component of CCA, effective schemes
are available by utilizing either extensions of the Oja's rule \cite{Oja1982SimplifiedNM} to the generalized eigenvalue problem \cite{bhatia2018gen} or the alternating SVRG algorithm \cite{gao2019stochastic}). 
Otherwise, a stochastic approach must make use of an explicit whitening operation 
which involves a cost of $d^3$ for each iteration \cite{arora2017stochastic}. 

{\bf Observation.}
Most approaches either directly optimize (\ref{cca_opt1}) or instead
a reparameterized or regularized form \cite{ge2016efficient,AllenZhu2016DoublyAM,arora2017stochastic}. Often,
the search space for $U$ and $V$ corresponds to
 the entire $\mathbf{R}^{d\times k}$ (ignoring the constraints for the moment). But if the formulation could be cast in a form which involved
approximately writing $U$ and $V$ as a product of structured matrices, we may be able to obtain specialized routines which are tailored to
exploit those properties. Such a reformulation is not difficult to derive -- where
the matrices used to express $U$ and $V$ can be identified
as objects that live in well studied geometric spaces. Then, utilizing the geometry of the space and borrowing relevant tools from
differential geometry leads to an efficient approximate algorithm for top-$k$ CCA which optimizes the population objective in a streaming fashion. 

\textbf{Contributions.}
(a) First, we re-parameterize the top-$k$ CCA problem as an optimization problem on specific
matrix manifolds, and show that it is equivalent to the original formulation in (\ref{cca_opt1}).
(b) Informed by the geometry of the manifold,
  we derive stochastic gradient descent (SGD) algorithms for solving the re-parameterized
  problem with $O(d^2k)$ cost per iteration
  and provide convergence rate guarantees.
  (c) This analysis gives a direct mechanism
  to obtain an upper bound on the number of
  iterations needed to guarantee an $\epsilon$ error w.r.t. the population objective for the
  CCA problem. 
  (d) The algorithm works in a streaming
  manner so it easily scales to large datasets and we do not need to assume access to
  the full dataset at the outset.
  (e) We present empirical evidence for both the standard CCA model
  and the DeepCCA setting \cite{andrew2013deep}, describing advantages
  and limitations. 
\section{Stochastic CCA: Reformulation, Algorithm and Analysis}
\vspace{-5pt}
\label{sec:method}
Let us recall the objective function for CCA as given in (\ref{cca_opt1}).
We denote $X\in \mathbf{R}^{N\times d_x}$ as
the matrix consisting of the samples $\left\{\mathbf{x}_i\right\}$ drawn
from a zero mean random variable $\bold{X}\sim\mathcal{X}$ and
$Y\in \mathbf{R}^{N\times d_y}$ denotes
the matrix consisting
of samples $\left\{\mathbf{y}_i\right\}$ drawn from a
zero mean random variable $\bold{Y}\sim\mathcal{Y}$.
For notational simplicity, we assume that
$d_x = d_y=d$ although the results hold for general $d_x$ and $d_y$. Also recall that $C_X$, $C_Y$ are the covariance matrices of $\bold{X}$, $\bold{Y}$. $C_{XY}$ is the cross-covariance matrix between $\bold{X}$ and $\bold{Y}$. $U\in\mathbf{R}^{d\times k}$ ($V\in \mathbf{R}^{d\times k}$) is the matrix consisting
of $\left\{\mathbf{u}_j\right\}$ ($\left\{\mathbf{v}_j\right\}$) , where $\left(\left\{\mathbf{u}_j\right\}, \left\{\mathbf{v}_j\right\}\right)$ are
the canonical directions. 
The constraints in (\ref{cca_opt1}) are called {\it whitening constraints}. 

\textbf{Reformulation:} In the CCA formulation, the matrices consisting of canonical correlation directions, i.e., $U$ and $V$, are unconstrained, hence the search space is the entire $\mathbf{R}^{d\times k}$. Now we reformulate the CCA objective by reparameterizing $U$ and $V$. In order to do that, let us take a brief detour and recall the objective function of principal component analysis (PCA):
\begin{align}
\label{pca_opt}
    \widehat{U} = \argmax_{U^{\prime}}
         &\quad \text{trace}(\widehat{R}) \qquad\qquad
     \text{subject to} & \quad \widehat{R} = U^{\prime T} C_X U^{\prime}; \:\:\:\: U^{\prime T}U^{\prime} = I_k
\end{align}

{\it Observe that by running PCA and assigning $U = \widehat{U}\widehat{R}^{\sfrac{-1}{2}}$ in (\ref{cca_opt1}) (analogous for $V$ using $C_Y$), we can satisfy the whitening constraint}. Of course, writing $U = \widehat{U}\widehat{R}^{\sfrac{-1}{2}}$ does satisfy the whitening constraint, but such a $U$ (and $V$) will not maximize $\text{trace}\left(U^T C_{XY}V\right)$, objective of  (\ref{cca_opt1}). Hence, additional work beyond the PCA solution is needed. Let us start from $\widehat{R}$
but relax the PCA solution 
by using  an arbitrary $\widetilde{R}$ instead of diagonal $\widehat{R}$ (this will still satisfy the whitening 
constraint). 

By writing $U = \widetilde{U}\widetilde{R}$ with  $\widetilde{U}^T\widetilde{U} = I_k$ and $\widetilde{R} \in \mathbf{R}^{k\times k}$. 
Thus we can approximate CCA objective (we will later show how good this approximation is) as
\begin{align}
\label{cca_opt3}
    \max_{\substack{ \widetilde{U}, \widetilde{V} \in \textsf{St}(k, d)\\ R_u, R_v \in \mathbf{R}^{k\times k}\\ U = \widetilde{U}R_u; ~V = \widetilde{V}R_v}}
        \: \underbrace{\text{trace}\left(U^T C_{XY}V \right)}_{\text{$\widetilde{F}$}} + \underbrace{\text{trace}\left(\widetilde{U}^TC_X\widetilde{U}\right) + \text{trace}\left(\widetilde{V}^TC_Y\widetilde{V}\right)}_{\text{$\widetilde{F}_{\text{pca}}$}}  & \: \:
    \text{s.t. } 
    & \:\: \substack{U^TC_XU = I_k \\ V^TC_YV = I_k}
\end{align}
Here, $\textsf{St}(k, d)$ denotes the manifold consisting of $d\times k$ (with $k\leq d$) column orthonormal matrices, i.e., $\textsf{St}(k, d) = \left\{X\in \mathbf{R}^{d\times k} | X^TX = I_k\right\}$. {\it Observe that in (\ref{cca_opt3}), we approximate the optimal $U$ and $V$ as a linear combination of $\widetilde{U}$ and $\widetilde{V}$ respectively. Thus, the aforementioned PCA solution can act as a feasible initial solution for (\ref{cca_opt3}}).  

As the choice of $R_u$ and $R_v$ is arbitrary, we can further reparameterize these matrices by constraining them to be full rank (of rank $k$) and using the RQ decomposition \cite{golub1971singular} which gives us the following reformulation.
\begin{mybox}{gray}{\bf A Reformulation for CCA}
\begin{subequations}
\label{cca_opt2}
\begin{align}
    \max_{\substack{ \widetilde{U}, \widetilde{V}, S_u, S_v, Q_u, Q_v\\ U = \widetilde{U}S_uQ_u; ~V = \widetilde{V}S_vQ_v}}
        & \quad \underbrace{\text{trace}\left(U^T C_{XY}V \right)}_{\text{$\widetilde{F}$}} + \underbrace{\text{trace}\left(\widetilde{U}^TC_X\widetilde{U}\right) + \text{trace}\left(\widetilde{V}^TC_Y\widetilde{V}\right)}_{\text{$\widetilde{F}_{\text{pca}}$}}   \label{cca_obj2}\\
    \text{subject to } 
    & \quad U^TC_XU = I_k \nonumber \\
    & \quad V^TC_YV = I_k  \label{cca_constraints2} \\ 
    & \quad \widetilde{U}, \widetilde{V} \in \textsf{St}(k, d);~Q_u, Q_v \in \textsf{SO}(k) \nonumber \\
    & \quad S_u, S_v ~\text{is upper triangular}\nonumber 
\end{align}
\end{subequations}
\end{mybox}

Here, $\textsf{SO}(k)$ is the space of $k\times k$ special orthogonal matrices, i.e., $\textsf{SO}(k) = \left\{X\in \mathbf{R}^{k\times k} | X^TX = I_k; \text{det}(X) = 1\right\}$. Before stating formally how good the aforementioned approximation is, we first point out some interesting properties of the reformulation (\ref{cca_opt2}): \begin{inparaenum}[\bfseries (a)] \item in the reparametrization of $U$ and $V$ all components are structured, hence, the search space becomes a subset of $\mathbf{R}^{k\times k}$ \item we can essentially initialize with a PCA solution and then optimize (\ref{cca_opt2}). \end{inparaenum}

{
\textbf{Why (\ref{cca_opt2}) helps?} First, we note that CCA seeks
to maximize the total correlation under the constraint that different components are decorrelated. 
The difficult part in the optimization is to ensure decorrelation, which leads to a higher complexity in existing streaming CCA algorithms. On the contrary, in (\ref{cca_opt2}), we separate (\ref{cca_opt1}) into finding the PCs, $\widetilde{U}, \widetilde{V}$ (by adding the variance maximization terms) and finding the linear combination ($S_uQ_u$ and $S_vQ_v$) of the principal directions. Thus, here we can (almost) utilize an efficient off-the-shelf streaming PCA algorithm. We will defer describing the 
specific  details of the optimization itself until the next sub-section. First, we will show formally why substituting (\ref{cca_opt1}) with (\ref{cca_opt2}) is sensible under some assumptions. 

\textbf{Why the solution of the reformulation makes sense?} 
We start by stating some mild assumptions needed for the analysis. \textbf{Assumptions:} 
\begin{inparaenum}[\bfseries(a)]
\item The random variables $\bold{X} \sim \mathcal{N}(\mathbf{0}, \Sigma_x)$ and $\bold{Y} \sim \mathcal{N}(\mathbf{0}, \Sigma_y)$ with $\Sigma_x \preceq cI_d$ and $\Sigma_y \preceq cI_d$ for some $c>0$.
\item The samples $X$ and $Y$ drawn from $\mathcal{X}$ and $\mathcal{Y}$ respectively have zero mean.
\item For a given $k\leq d$, $\Sigma_x, \Sigma_y$ have non-zero top-$k$ eigen values.
\end{inparaenum}

We show how the presented solution, assuming access to an effective numerical procedure,
approximates the CCA problem presented in (\ref{cca_opt1}).
We formally state the result in the following theorem with a sketch of proof (appendix includes the full proof) {by first stating the following proposition}.

\begin{definition}
\label{def1}
A random variable $\bold{X}$ is called sub-Gaussian if the norm given by 
$\|\bold{X}\|_{\star} := \inf\left\{d\geq 0 | \mathbf{E}_{\bold{X}}\left[\exp\left(\sfrac{\text{trace}(X^TX)}{d^2}\right)\right]\leq 2\right\}$ is finite. Let $U \in \mathbf{R}^{d\times k}$, then $\bold{X}U$ is sub-Gaussian \cite{vershynin2017four}. 
\end{definition}
}

\begin{proposition}[\cite{reiss2020nonasymptotic}]
\label{prop2}
Let $\bold{X}$ be a random variable which
follows a sub-Gaussian distribution.
Let $\widehat{X}$ be the approximation of $X \in \mathbf{R}^{N\times d}$ (samples drawn from $\mathcal{X}$)
with the top-$k$ principal vectors.
Let $\widetilde{C}_X$ be the covariance of $\widehat{X}$.
Also, assume that $\lambda_i$ is the $i^{th}$ eigen value of $C_X$ for {$i = 1, \cdots, d-1$} and $\lambda_i \geq \lambda_{i+1}$ for all $i$.
Then, the PCA reconstruction error, denoted by $\mathcal{E}_k = \|X-\widehat{X}\|$ (in the Frobenius norm sense) can be upper bounded as follows
\begin{align*}
\mathcal{E}_k \leq \min\left(\sqrt{2k}\|\Delta\|_2, \frac{2\|\Delta\|_2^2}{\lambda_k - \lambda_{k+1}}\right), \quad \mbox{ $\Delta = C_X - \widetilde{C}_X$.}
\end{align*}
\end{proposition}

{
The aforementioned proposition suggests that the error between the data matrix $X$ and the reconstructed data matrix $\widehat{X}$ using the top-$k$ principal vectors is bounded.

Recall from (\ref{cca_opt1}) and (\ref{cca_opt2}) that the optimal value of the true and approximated CCA objective is denoted by $F$ and $\widetilde{F}$ respectively. The following theorem states that we can bound the error, $E=\|F-\widetilde{F}\|$ (proof in the appendix). In other words, if we start from PCA solution and can successfully optimize (\ref{cca_opt2}) without 
leaving the feasible set, we will 
obtain a good solution. 
}

\begin{theorem}
\label{prop3}
Using the hypothesis and assumptions above, the approximation error $E = \|F-\widetilde{F}\|$ is bounded and goes to zero while the {\it whitening constraints} in \eqref{cca_constraints2} are satisfied.
\end{theorem}

\begin{proof}[Sketch of the Proof]
Let $U^*$ and $V^*$ be the true solution of CCA, i.e., of (\ref{cca_opt1}). Let $U = \widetilde{U}S_uQ_u,  V = \widetilde{V}S_vQ_v$ be the solution of (\ref{cca_opt2}), with $\widetilde{U}, \widetilde{V}$ be the PCA solutions of $X$ and $Y$ respectively. Let $\widehat{X} = X\widetilde{U}\widetilde{U}^T$ and $\widehat{Y} = Y\widetilde{V}\widetilde{V}^T$ be the reconstruction of $X$ and $Y$ using principal vectors. Let $S_uQ_u = \widetilde{U}^TU^*$ and $S_vQ_v = \widetilde{V}^TV^*$. Then we can write $\widetilde{F} = \text{trace}\left(U^TC_{XY}V\right)$ $= \text{trace}\left(\frac{1}{N} \left(\widehat{X}U^*\right)^T\widehat{Y}V^*\right)$. Similarly we can write $F = \text{trace}\left(\frac{1}{N} \left(XU^*\right)^TYV^*\right)$. As, $\widehat{X}$ and $\widehat{Y}$ are the approximation of $X$ and $Y$ respectively using principal vectors, we use proposition \ref{prop2} to bound the error  $\|F - \widetilde{F}\|$. Now observe that $\widehat{X}U$ can be rewritten into $X\widetilde{U}\widetilde{U}^TU$ (similar for $\widehat{Y}V$). Thus, as long as the solution $S_uQ_u$ and $S_vQ_v$ respectively well-approximate $\widetilde{U}^TU$ and $\widetilde{V}^TV$,  $\widetilde{F}$ is a good approximation of $F$. 
\end{proof}

Now, the only unresolved issue is an optimization scheme for
\eqref{cca_obj2} that {keeps the constraints in \eqref{cca_constraints2} satisfied by leveraging the geometry of the structured solution space}. 

\vspace{-5pt}
\subsection{How to numerically optimize (\ref{cca_obj2}) satisfying constraints in (\ref{cca_constraints2})?}
{\bf Overview.} We now describe how to maximize the formulation in (\ref{cca_obj2})--(\ref{cca_constraints2}) with respect
to $\widetilde{U}$, $\widetilde{V}$, $Q_u$, $Q_v$, $S_u$ and $S_v$. 
%
%
We will
first compute top-$k$ principal vectors to get $\widetilde{U}$ and $\widetilde{V}$. Then, 
we will use a gradient update rule to solve for $Q_u$, $Q_v$, $S_u$ and $S_v$ {to improve the objective}.
Since all these matrices are ``structured'', 
care must be taken to ensure that the matrices {\em remain on their respective manifolds} -- which is where 
the geometry of the manifolds will offer desirable properties. 
We re-purpose a Riemannian stochastic gradient descent (RSGD) to do this,
so call our algorithm {\it RSG+}. Of course,
more sophisticated Riemannian optimization techniques can be
substituted in. {For instance, different Riemannian optimization methods are available in \cite{Absil2007OptimizationAO} and optimization schemes for many manifolds are offered in PyManOpt \cite{boumal2014manopt}.} 

The algorithm block is presented in Algorithm \ref{sec3:alg2}.  
{Recall $\widetilde{F}_{\text{pca}} = \text{trace}\left(U^TC_XU)\right) + \text{trace}\left(V^TC_YV)\right)$ be the 
contribution from the principal directions which we used to ensure the ``whitening constraint''. Moreover, $\widetilde{F} = \text{trace}\left(U^TC_{XY}V\right)$ be the contribution from the canonical correlation directions ({\it note that we use the subscript 'cca' for making CCA objective explicit}).} 
The algorithm consists of four main blocks denoted by different colors,
namely
\begin{inparaenum}[\bfseries (a)]
\item the {\color{red!70} Red} block {deals with gradient calculation of the objective function where we calculate the top-$k$ principal vectors (denoted by $\widetilde{F}_{\text{pca}}$) with respect to $\widetilde{U}$, $\widetilde{V}$};
\item the {\color{green!70} Green} block describes
  calculation of the gradient {corresponding to the canonical directions (denoted by $\widetilde{F}$) with respect to $\widetilde{U}$, $\widetilde{V}$, $S_u$, $S_v$, $Q_u$ and $Q_v$};
\item the {\color{gray!70} Gray} block {combines the gradient computation from both $\widetilde{F}_{\text{pca}}$ and $\widetilde{F}$ with respect to unknowns $\widetilde{U}$, $\widetilde{V}$, $S_u$, $S_v$, $Q_u$ and $Q_v$}; and finally
\item the {\color{blue!70} Blue} block performs a
  batch update of the canonical directions $\widetilde{F}$ {using Riemannian gradient updates}.
\end{inparaenum}

{\bf Gradient calculations.} The gradient update for $\widetilde{U}, \widetilde{V}$
is divided into two parts
\begin{inparaenum}[\bfseries (a)]
\item The ({\color{red!70} Red} block) gradient updates the ``principal''
  directions (denoted by $\nabla_{\widetilde{U}} {\widetilde{F}_{\text{pca}}}$ and $\nabla_{\widetilde{V}} {\widetilde{F}_{\text{pca}}}$),
  which is specifically designed to satisfy the {\it whitening constraint}. {Since this requires updating the principal subspaces, so, the gradient descent needs to proceed on the manifold of $k$-dimensional subspaces of $\mathbf{R}^d$, i.e., on the Grassmannian  $\textsf{Gr}(k, d)$.}
\item The ({\color{green!70} green} block) gradient
  from the objective function in (\ref{cca_opt2}),
  is denoted by $\nabla_{\widetilde{U}} {\widetilde{F}}$ and $\nabla_{\widetilde{V}} {\widetilde{F}}$.
\end{inparaenum}
{In order to ensure that the Riemannian gradient update for $\widetilde{U}$ and $\widetilde{V}$ stays on the manifold $\textsf{St}(k,d)$, we need to make sure that the gradients, i.e., $\nabla_{\widetilde{U}} {\widetilde{F}}$ and $\nabla_{\widetilde{V}} {\widetilde{F}}$ lies in the tangent space of $\textsf{St}(k,d)$.}  
{In order to do that, we need to first calculate the Euclidean gradient and then project on to the tangent space of $\textsf{St}(k,d)$.}

The gradient updates for $Q_u, Q_v, S_u, S_v$ are given in the {\color{green!70} green} block, denoted by $\nabla_{Q_u} {\widetilde{F}}$, $\nabla_{Q_v} {\widetilde{F}}$, $\nabla_{S_u} {\widetilde{F}}$ and $\nabla_{S_v} {\widetilde{F}}$. Note that unlike  the previous step, this gradient only has components from {canonical correlation computation}. As before, this step
requires first computing the Euclidean gradient and then projecting  on to the tangent space of the underlying Riemannian manifolds involved, {i.e., $\textsf{SO}(k)$ and the space of upper triangular matrices}.

Finally, we get the gradient to update the canonical directions by combining the gradients which is shown in {\color{gray!70} gray} block. With these gradients we can perform a
batch update as shown in the {\color{blue!70} blue} block. A schematic diagram is given in Fig. \ref{cca_flow}.

\begin{wrapfigure}[15]{r}{0.5\columnwidth}
  \begin{center}
    \includegraphics[width=0.5\columnwidth]{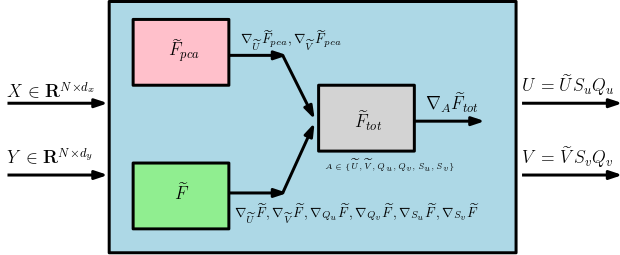}
  \end{center}
  \caption{Schematic diagram of the proposed CCA algorithm, here $\widetilde{F}_{\text{tot}} = \widetilde{F} +  \widetilde{F}_{\text{pca}}$, where $\widetilde{F}$ is the approximated objective value for CCA (as in (\ref{cca_opt2}))}
  \label{cca_flow}
\end{wrapfigure}
%
%
Using convergence
results presented next in Propositions \ref{prop4}--\ref{prop5}, this scheme can be shown
(under some assumptions) to approximately optimize the CCA objective in (\ref{cca_opt1}).

We can now move to the convergence properties of the algorithm.
We present two results
stating the asymptotic proof of convergence for 
{\color{red!30} top-$k$ principal vectors} and
{\color{green!30} canonical directions} in the algorithm.
\begin{proposition}[\cite{chakraborty2020intrinsic}]
\label{prop4}
(Asymptotically) If the samples, $X$, are drawn from a Gaussian distribution, then the gradient update rule  presented in Step \ref{gr_update} in Algorithm \ref{sec3:alg2} returns an
orthonormal basis -- the top-$k$ principal vectors of the covariance matrix $C_X$.
\end{proposition}

\begin{proposition}(\cite{bonnabel2013stochastic})
\label{prop5}
Consider a connected Riemannian manifold $\mathcal{M}$ with injectivity radius bounded from below by $I >0$. Assume that
the sequence of step sizes $\left(\gamma_l\right)$ satisfy the condition \begin{inparaenum}[(a)] \item $\sum \gamma_l^2 < \infty$ \item $\sum \gamma_l = \infty$\end{inparaenum}. Suppose $\left\{A_l\right\}$ lie in a compact set $K \subset \mathcal{M}$. We also suppose that $\exists D >0$ such that, $g_{A_l}\left(\nabla_{A_l} {\widetilde{F}}, \nabla_{A_l}\ {\widetilde{F}}\right)\leq D$. Then $\nabla_{A_l}\ {\widetilde{F}} \rightarrow 0$ and $l\rightarrow \infty$. 
\end{proposition}

Notice that in our problem, the injectivity radius bound in Proposition \ref{prop5} is satisfied as ``$I$'' for  $\textsf{Gr}(p,n)$, $\textsf{St}(p,n)$ or $\textsf{SO}(p)$ is $\pi/2\sqrt{2}, \pi/2\sqrt{2}, \pi/2$ respectively. So, in order to apply Proposition \ref{prop5}, we need to guarantee the step sizes satisfy the aforementioned condition.
One example of the step sizes that satisfies the
property is $\gamma_l = \frac{1}{l+1}$. 

\vspace{-5pt}
\subsection{Convergence rate and complexity of the RSG+ algorithm}
\vspace{-5pt}
In this section, we describe
the convergence rate and complexity of the algorithm proposed
in Algorithm \ref{sec3:alg2}. Observe that the key component of Algorithm \ref{sec3:alg2}
is a Riemannian gradient update. Let $A_t$ be the generic entity needed to be updated in
the algorithm using the Riemannian gradient update
$A_{t+1} = \textsf{Exp}_{A_t}\left(-\gamma_t \nabla_{A_t}\widetilde{F}\right)$,
where $\gamma_t$ is the step size at time step $t$.
Also assume $\left\{A_t\right\}\subset \mathcal{M}$ for a Riemannian manifold $\mathcal{M}$.
The following proposition states that under certain assumptions,
the Riemannian gradient update has a convergence rate of $O\left(\frac{1}{t}\right)$.

\begin{algorithm}
{\small
    \textbf{Input: }{$X \in \mathbf{R}^{N\times d_x}$, $Y \in \mathbf{R}^{N\times d_y}$, $k > 0$}\\
    \textbf{Output: }{$U \in \mathbf{R}^{d_x\times k}$, $V \in \mathbf{R}^{d_y\times k}$}
    
    Initialize $\widetilde{U}, \widetilde{V}, Q_u, Q_v, S_u, S_v$ \\
    Partition data $X, Y$ into batches of size $B$. Let $j^{th}$ batch be denoted by $X_j$ and $Y_j$ \\
    \textbf{for  }$j \in \left\{1, \cdots, \lfloor\frac{N}{B}\rfloor \right\}$ \textbf{  do}\\
           \begin{mybox}{red}{{\bf Gradient for top-$k$ principal vectors}: calculating $\nabla_{\widetilde{U}} \widetilde{F}_{\text{pca}},\nabla_{\widetilde{V}} \widetilde{F}_{\text{pca}}$}
             \begin{compactenum}
             \item Partition $X_j$ ($Y_j$) into $L$ ($L=\lfloor\frac{B}{k}\rfloor$) blocks  of size $d_x\times k$ ($d_y\times k$)\;
             \item Let the $l^{th}$ block be denoted by $Z^x_l$ ($Z^y_l$)\;
             \item Orthogonalize each block and let the orthogonalized block be denoted by $\hat{Z}^x_l$ ($\hat{Z}^y_l$)\;
             \item Let the subspace spanned by each $\hat{Z}^x_l$ (and $\hat{Z}^y_l$) be $\mathcal{\hat{Z}}^x_l \in \textsf{Gr}(k, d_x)$ (and $\mathcal{\hat{Z}}^y_l \in \textsf{Gr}(k, d_y)$)\;
             \end{compactenum}
             {\small 
               \begin{flalign}
               \label{gr_update}
               \nabla_{\widetilde{U}} {\widetilde{F}_{\text{pca}}} = -\sum_{l} \textsf{Exp}^{-1}_{\widetilde{U}}\left(\mathcal{\hat{Z}}^x_l\right) \quad 
               \nabla_{\widetilde{V}} {\widetilde{F}_{\text{pca}}} = -\sum_{l} \textsf{Exp}^{-1}_{\widetilde{V}}\left(\mathcal{\hat{Z}}^y_l\right) 
             \end{flalign}
             }
           \end{mybox}
            
           \begin{mybox}{green}{{\bf Gradient from \eqref{cca_opt2}}:  calculating
               $\nabla_{\widetilde{U}}\widetilde{F},\nabla_{\widetilde{V}}\widetilde{F},\nabla_{Q_u}\widetilde{F},\nabla_{Q_v}\widetilde{F},\nabla_{S_u}\widetilde{F},\nabla_{S_v}\widetilde{F}$}
           $\nabla_{\widetilde{U}}\widetilde{F} = \frac{\partial{\widetilde{F}}}{\partial \widetilde{U}} - \widetilde{U}\frac{\partial {\widetilde{F}}}{\partial \widetilde{U}}^T\widetilde{U} \quad\nabla_{\widetilde{V}}\widetilde{F} = \frac{\partial {\widetilde{F}}}{\partial \widetilde{V}} - \widetilde{V}\frac{\partial {\widetilde{F}}}{\partial \widetilde{V}}^T\widetilde{V}$\\
           $\nabla_{Q_u}\widetilde{F} = \frac{\partial {\widetilde{F}}}{\partial Q_u} -  \frac{\partial {\widetilde{F}}}{\partial Q_u}^T \quad \nabla_{Q_v}\widetilde{F}= \frac{\partial {\widetilde{F}}}{\partial Q_v} -  \frac{\partial {\widetilde{F}}}{\partial Q_v}^T$\\
           $\nabla_{S_u}\widetilde{F}= \text{Upper}\left(\frac{\partial {\widetilde{F}}}{\partial S_u}\right) \quad \nabla_{S_v}\widetilde{F}= \text{Upper}\left(\frac{\partial {\widetilde{F}}}{\partial S_v}\right)$\\
           
             Here, $\text{Upper}$ returns the upper triangular matrix of the input matrix and
             $\frac{\partial{\widetilde{F}}}{\partial \widetilde{U}}, \frac{\partial{\widetilde{F}}}{\partial \widetilde{V}}, \frac{\partial{\widetilde{F}}}{\partial Q_u}, \frac{\partial{\widetilde{F}}}{\partial Q_v}, \frac{\partial{\widetilde{F}}}{\partial S_u}, \frac{\partial{\widetilde{F}}}{\partial S_v}$ give the Euclidean gradients, which are provided in appendix.
            \end{mybox}
            
             \begin{mybox}{gray}{\bf Gradient to update canonical directions}
              $\nabla_{\widetilde{U}}\widetilde{F}_{\text{tot}} = \nabla_{\widetilde{U}}\widetilde{F}_{\text{pca}} + \nabla_{\widetilde{U}}\widetilde{F} \qquad \qquad \nabla_{\widetilde{V}}\widetilde{F}_{\text{tot}} = \nabla_{\widetilde{V}}\widetilde{F}_{\text{pca}} + \nabla_{\widetilde{V}}\widetilde{F}$\;
             $\nabla_{X}\widetilde{F}_{\text{tot}}  =\nabla_{X}\widetilde{F}$ where, $X$ is a generic entity: $X\in \{ Q_u, Q_v, S_u, S_v\}$\;               
             \end{mybox}
             
             \begin{mybox}{blue}{\bf Batch update of canonical directions}
              ${A} = \textsf{Exp}_{{A}}\left(-\gamma_j \nabla_{{A}}\widetilde{F}_{\text{tot}}\right)$ where, $A$ is a generic entity: $A\in \{\widetilde{U}, \widetilde{V}, Q_u, Q_v, S_u, S_v\}$\;               
             \end{mybox}
	\textbf{end for}
     
    $U = \widetilde{U}Q_uS_u$ and $V = \widetilde{V}Q_vS_v$\;
    \caption{\footnotesize{Riemannian SGD based algorithm (RSG+) to compute canonical directions}}
    \label{sec3:alg2}  
    }

\end{algorithm}

\begin{proposition}(\cite{Nemirovski2009RobustSA,becigneul2018riemannian})
\label{prop6}
Let $\left\{A_t\right\}$ lie inside a geodesic ball of radius less than
the minimum of the {\it injectivity radius} and the
{\it strong convexity radius} of $\mathcal{M}$. Assume $\mathcal{M}$ to
be a geodesically complete Riemannian manifold with sectional curvature
lower bounded by $\kappa\leq 0$. Moreover, assume that the step size
$\left\{\gamma_t\right\}$ diverges and the squared step size converges.
Then, the Riemannian gradient descent update given by
$A_{t+1} = \textsf{Exp}_{A_t}\left(-\gamma_t \nabla_{A_t}\widetilde{F}\right)$ with a
bounded $\nabla_{A_t}\widetilde{F}$, i.e., $\|\nabla_{A_t}\widetilde{F}\|\leq C<\infty$ for
some $C\geq 0$, converges in the rate of $O\left(\frac{1}{t}\right)$.
\end{proposition}
\vspace{-5pt}

{All Riemannian manifolds we used, i.e., $\textsf{Gr}(k, d)$, $\textsf{St}(k, d)$ and $\textsf{SO}(k)$ are geodesically complete, and these manifolds have non-negative sectional curvatures, i.e., lower bounded by  $\kappa=0$. Moreover the minimum of convexity and injectivity radius for  $\textsf{Gr}(k, d)$, $\textsf{St}(k, d)$  and $\textsf{SO}(k)$ are $\pi/2\sqrt{2}$. Now, as long as the Riemannian updates lie inside the geodesic ball of radius less than $\pi/2\sqrt{2}$, the convergence rate for RGD applies in our setting.}

{\bf Running time.} To evaluate time complexity, we must look at the main compute-heavy steps
needed. The basic modules are $\textsf{Exp}$ and
$\textsf{Exp}^{-1}$ maps for $\textsf{St}(k, d)$, $\textsf{Gr}(k, d)$ and $\textsf{SO}(k)$ manifolds (see appendix).
Observe that the complexity of these modules
is influenced by the complexity of $\textsf{svd}$
needed for the $\textsf{Exp}$ map for the $\textsf{St}$ and $\textsf{Gr}$ manifolds. 
{Our algorithm involves structured matrices of size $d\times k$ and $k\times k$, so any matrix operation should not exceed a cost of $O(\max(d^2k, k^3))$, since in general $d \gg k$. Specifically, the most expensive calculation is SVD of matrices of size $d\times k$, which is $O(d^2k)$, see \cite{golub1971singular}. All other calculations are dominated by this term.}

\vspace{-10pt}
\section{Experiments}
\vspace{-10pt}
We first evaluate RSG+ for extracting top-$k$ canonical components on three benchmark datasets and show that it performs favorably compared with \cite{arora2017stochastic}. Then, we show that RSG+ also fits into feature learning in DeepCCA  \cite{andrew2013deep}, and can scale to large feature dimensions where the non-stochastic method fails. Finally we show that RSG+ can be used to improve fairness of deep neural networks without full access to labels of protected attributes during training.

\vspace{-6pt}
\subsection{CCA on Fixed Datasets}
\vspace{-5pt}
\textbf{Datasets and baseline.} We conduct experiments on three benchmark datasets (MNIST \cite{lecun2010mnist}, Mediamill \cite{snoek2006challenge} and CIFAR-10 \cite{Krizhevsky09learningmultiple}) to evaluate the performance of RSG+ to extract top-$k$ canonical components. To our best knowledge, \cite{arora2017stochastic} is the only previous work which stochastically optimizes the population objective in a streaming fashion and can extract top-$k$ components, so we compare our RSG+ with the matrix stochastic gradient (MSG) method proposed in \cite{arora2017stochastic} (There are two methods proposed in \cite{arora2017stochastic} and we choose MSG because it performs better in the experiments in \cite{arora2017stochastic}). The details regarding the three datasets and how we process them are as follows:

{\bf MNIST} \cite{lecun2010mnist}: MNIST contains grey-scale images of size $28\times 28$. We use its full training set containing $60$K images. Every image is split into left/right half, which are used as the two views.
{\bf Mediamill} \cite{snoek2006challenge}: Mediamill contains around $25.8$K paired features of videos and corresponding commentary of dimension $120,101$ respectively.
{\bf CIFAR-10} \cite{Krizhevsky09learningmultiple}: CIFAR-10 contains $60$K $32\times 32$ color images. Like MNIST, we split the images into left/right half and use them as two views. 

\textbf{Evaluation metric.} We choose to use Proportion of Correlations Captured (PCC) which is widely used  \cite{Ma2015FindingLS,ge2016efficient}, partly due to its efficiency, especially for relatively large datasets. Let $\hat{U}\in R^{d_x\times k}, \hat{V}\in R^{d_y\times k}$ denote the estimated subspaces returned by RSG+, and $U^*\in R^{d_x\times k}, V^*\in R^{d_y\times k}$ denote the true canonical subspaces (all for top-$k$). The PCC is defined as $
	\text{PCC} = \frac{\text{TCC}(X\hat{U}, Y\hat{V})}{\text{TCC}(XU^*, YV^*)}
$, 
where TCC is the sum of canonical correlations between two matrices.



\textbf{Performance.}  The performance in terms of PCC as a function of \# of seen samples (coming in a streaming way) are shown in Fig. \ref{fig_all}, and our RSG+ achieves around 10 times runtime improvement from MSG (see appendix for the table). 
Our RSG+ captures more correlation than MSG \cite{arora2017stochastic} while being $5-10$ times faster. One case where our RSG+ underperforms \cite{arora2017stochastic} is when the top-$k$ eigenvalues are dominated by the top-$l$ eigenvalues with $l<k$ (Fig. \ref{fig:mediamill_all}): on Mediamill dataset, the top-4 eigenvalues of the covariance matrix in view 1 are: $8.61,2.99,1.15,0.37$. The first eigenvalue is dominantly large compared with the rest and our RSG+ performs better for $k=1$ and worse than \cite{arora2017stochastic} for $k=2,4$. 
We also provide the runtime of RSG+ under different data dimension (set $d_x=d_y=d$) and number of total samples sampled from joint gaussian distribution in appendix. 

\begin{figure*}[!ht]
\centering
\begin{subfigure}[t]{.33\textwidth}
  \centering
  \includegraphics[width=.95\linewidth]{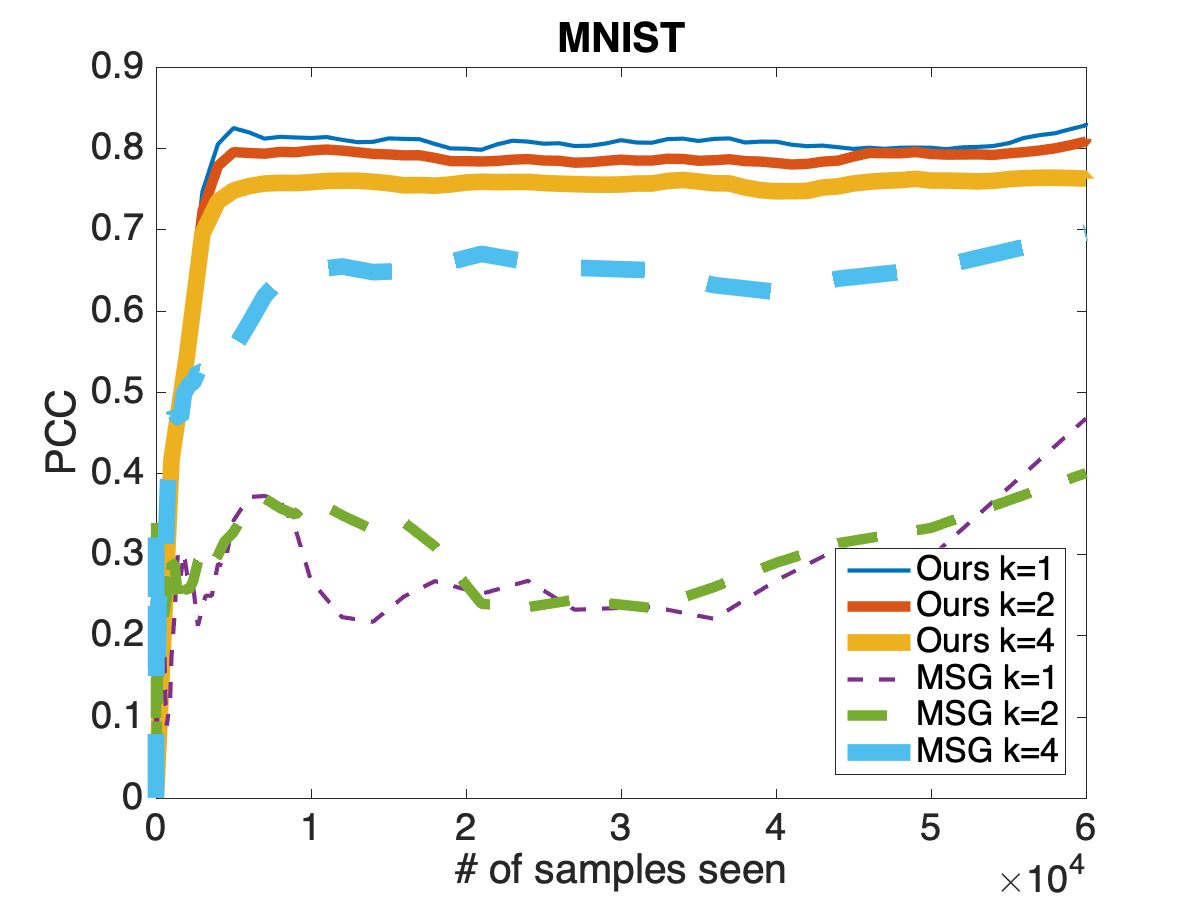}
  \caption{on MNIST}
  \label{fig:mnist_all}
\end{subfigure}%
\begin{subfigure}[t]{.33\textwidth}
  \centering
  \includegraphics[width=.95\linewidth]{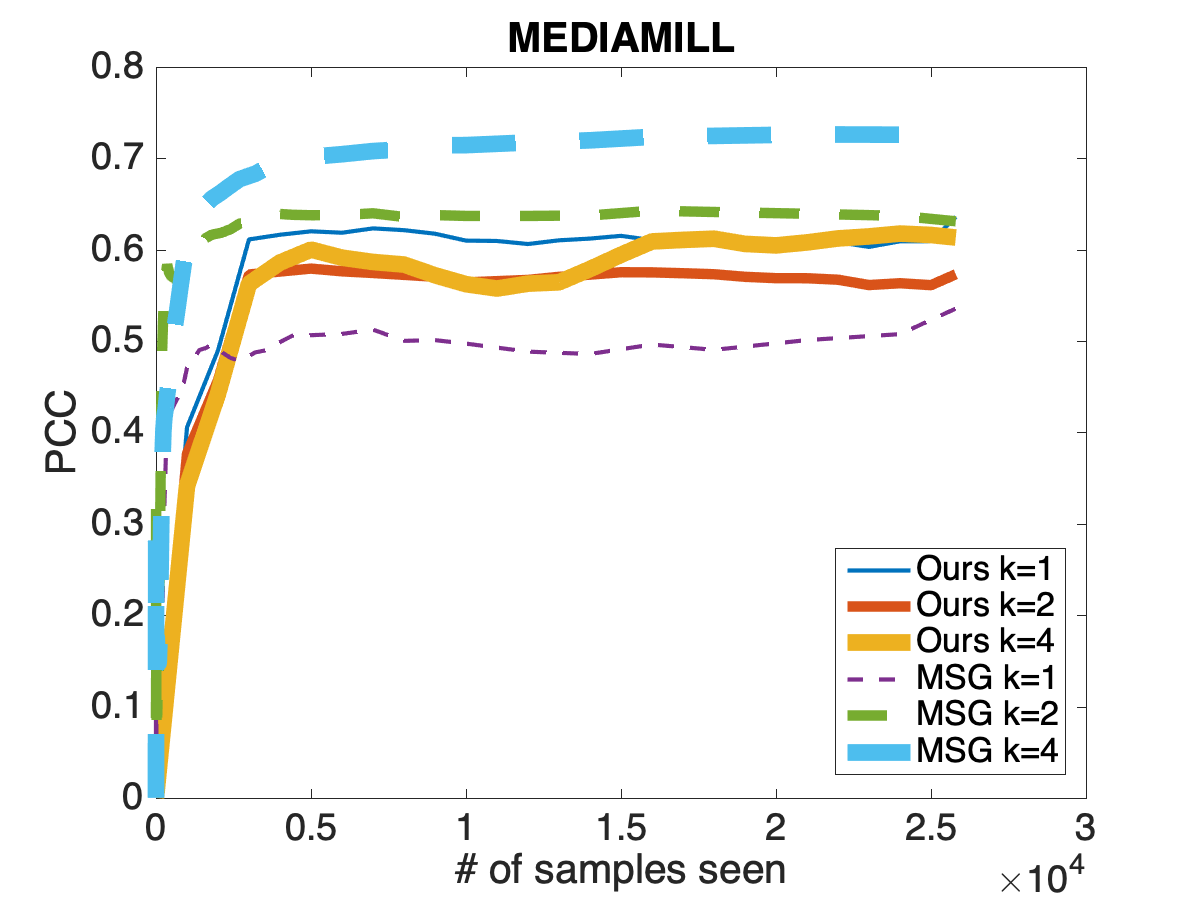}
  \caption{on Mediamill}
  \label{fig:mediamill_all}
\end{subfigure}
\begin{subfigure}[t]{.33\textwidth}
  \centering
  \includegraphics[width=.95\linewidth]{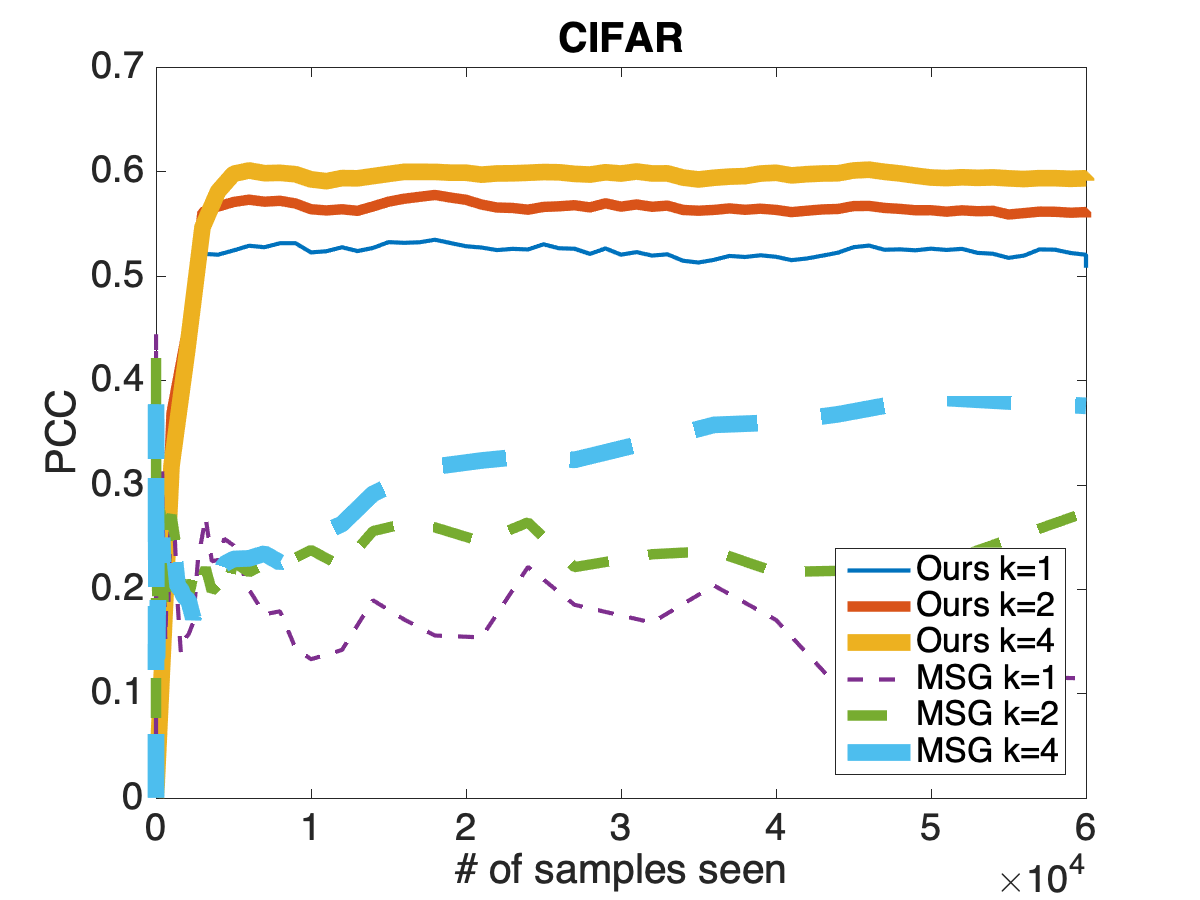}
  \caption{on CIFAR}
  \label{fig:cifar_all}
\end{subfigure}
\caption{Performance on three datasets in terms of PCC as a function of \# of seen samples.}
\label{fig_all}
\end{figure*}

\subsection{CCA for Deep Feature Learning}

\begin{wraptable}[8]{r}{0.50\columnwidth}
{

\caption{Results of feature learning on MNIST. N/A means fails to yield a result on our hardware.}
\resizebox{.5\columnwidth}{!}{
\centering

\begin{tabular}{l  c c c }
\hline
 Accuracy(\%) & $d=100$  &$d=500$  &$d=1000$  \\
\hline
DeepCCA    & $80.57$  &  $N/A$  &  $N/A$    \\
Ours   & $79.79$  &  $84.09$  &  $86.39$ \\ \hline
\label{DCCA}
\end{tabular}
}

}
\end{wraptable}

\textbf{Background and motivation.} A deep neural network (DNN) extension of CCA was proposed by \cite{andrew2013deep} and has become popular in multi-view representation learning tasks. The idea is to learn a deep neural network as the mapping from original data space to a latent space where the canonical correlations are maximized. We refer the reader
to \cite{andrew2013deep} for details of the task. Since deep neural networks are usually trained using SGD on mini-batches, this requires getting an estimate of the CCA objective at every iteration in a streaming fashion, thus our RSG+ can be a natural fit. We conduct experiments on a noisy version of MNIST dataset to evaluate RSG+.

\textbf{Dataset.} We follow \cite{wang2015deep} to construct a noisy version of MNIST: View $1$ is a randomly sampled image which is first rescaled to $[0,1]$ and then rotated by a random angle from $[-\frac{\pi}{4}, \frac{\pi}{4}]$. View $2$  is randomly sampled from the same class as view $1$. Then we add independent uniform noise from $[0,1]$ to each pixel. Finally the image is truncated into $[0,1]$ to form the view $2$.

\textbf{Implementation details.} We use a simple $2$-layer MLP with ReLU nonlinearity, where the hidden dimension in the middle is $512$ and the output feature dimension is $d\in\{100,500,1000\}$. After the network is trained on the CCA objective, we use a linear Support Vector Machine (SVM) to measure classification accuracy on output latent features. \cite{andrew2013deep} uses the closed form CCA objective on the current batch directly, which costs $O(d^3)$ memory and time for every iteration.

\textbf{Performance.} Table \ref{DCCA} shows that we get similar performance when $d=100$ and can scale to large latent dimensions $d=1000$ while the batch method \cite{andrew2013deep} encounters numerical difficulty on our GPU resources and the Pytorch \cite{paszke2019pytorch} platform in performing an eigen-decomposition of a $d\times d$ matrix when $d=500$, and becomes difficult if $d$ is larger than $1000$.

\subsection{CCA for Fairness Applications}

\begin{wrapfigure}{r}{0.7\columnwidth}

    \centering
    \includegraphics[width=0.7\columnwidth]{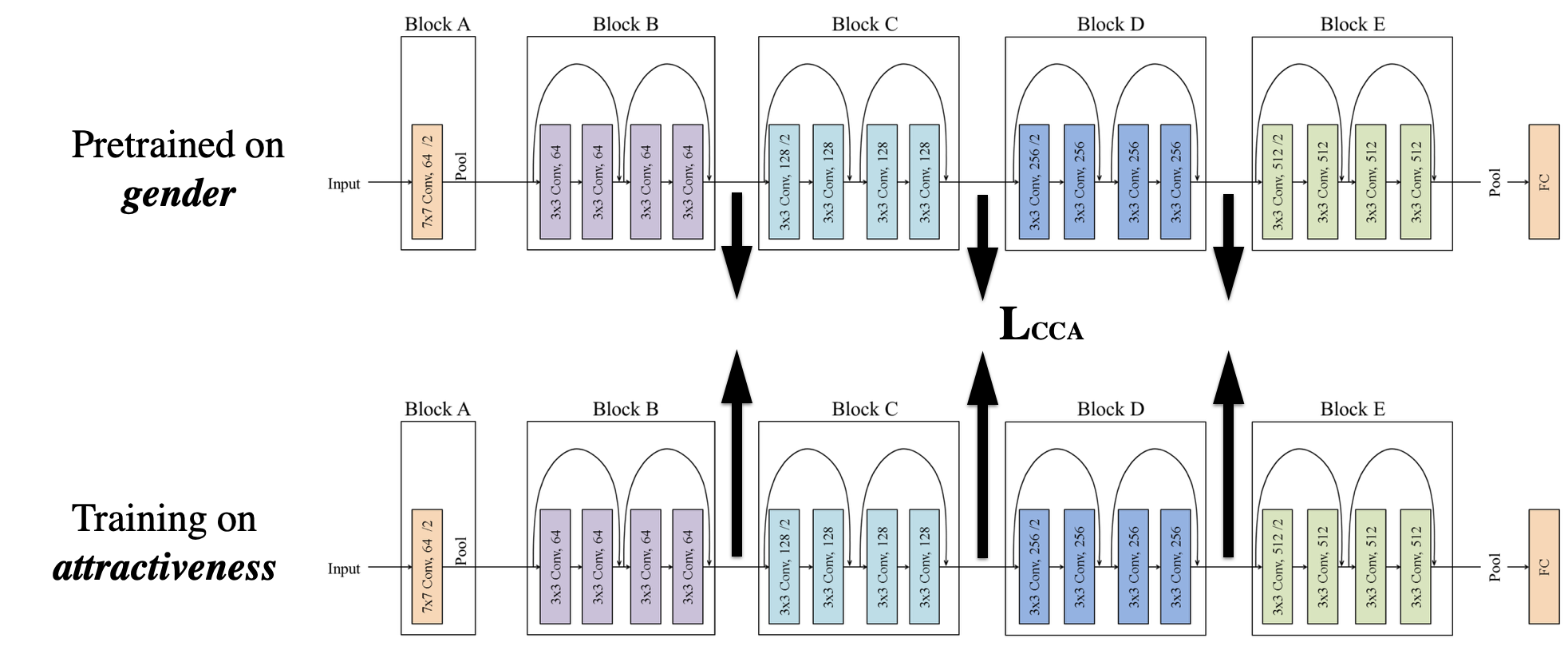}
    \caption{Training architecture for fairness experiment. The model above is the pretrained model and the model below is being trained. Use of CCA allows the two network architectures to be different.}
    \label{fig:cca_explain}
    
\end{wrapfigure}

\textbf{Background and motivation.} Fairness is becoming an important issue to consider in the design of learning algorithms. A common 
strategy to make an algorithm fair is to remove the influence of one/more protected attributes  when training the 
models, see \cite{lokhande2020fairalm}. Most methods assume that the labels of protected attributes are known during training but this may not always be possible. CCA  enables considering a slightly different setting, where we may not have per-sample protected attributes which may be sensitive or hard to obtain for third-parties \cite{price2019privacy}. On the other hand, we assume that a model pre-trained to predict the protected attribute labels is provided. For example, if the protected attribute is gender, we only assume that a good classifier which is trained to predict gender from the samples is available rather than sample-wise gender values themselves. We next demonstrate that fairness of the model, using standard measures, can be improved via constraints on correlation values from CCA.

\textbf{Dataset.} CelebA \cite{wang2015stochastic} consists of $200$K celebrity face images from the internet. There are up to $40$ labels, each of which is binary-valued. Here, we follow \cite{lokhande2020fairalm} to focus on the {\it attactiveness} attribute (which we want to train a classifier to predict) and the {\it gender} is treated as  ``protected'' since it may lead to an unfair classifier according to  \cite{lokhande2020fairalm}.

\textbf{Method.} Our strategy is inspired by \cite{morcos2018insights} which showed that canonical correlations can reveal the similarity in neural networks: when two networks (same architecture) are trained using different labels/schemes for example, canonical correlations can indicate how similar their features are. Our observation is the following. Consider a classifier that is trained on gender (the protected attribute), and another classifier that is trained on {\it attractiveness}, if the features extracted by the latter model share a high similarity with the one trained to predict gender, then it is more likely that the latter model is influenced by features in the image pertinent to gender, which will lead to an unfairly biased trained model. We show that by imposing a loss on the canonical correlation between the network being trained (but we lack per-sample protected attribute information) and a well trained classifier pre-trained on the protected attributes, we can obtain a more fair model. This may enable training fairer models in settings which would otherwise be difficult.
The training architecture is shown in Fig. \ref{fig:cca_explain}.

\begin{wraptable}[15]{r}{0.6\columnwidth}
{

\caption{Fairness results on CelebA. We applied CCA on three different layers in Resnet-18 respectively. See  appendix for  positions of conv $0,1,2$. ``Ours-conv[0,1]-conv[1,2]'' means stacking features from different layers to form hypercolumn  features \cite{hariharan2015hypercolumns}, which shows that our approach allows two networks to have different shape/size.}
\resizebox{.6\columnwidth}{!}{
\centering
\begin{tabular}{l  c c c }
\hline
  & Accuracy(\%)  & DEO(\%) &  DDP(\%)\\
\hline
Unconstrained   & $76.3$  &  $22.3$  &  $4.8$    \\
Ours-conv$0$   & $76.5$  &  $17.4$  &  $\textbf{1.4}$ \\ 
Ours-conv$1$   & $\textbf{77.7}$  &  $\textbf{15.3}$  &  $3.2$ \\
Ours-conv$2$   & $75.9$  &  $22.0$  &  $2.8$ \\
Ours-conv[0,1]-conv[1,2] & 76.0 & 22.1 & 3.9\\
\hline
\label{table_fair}
\end{tabular}
}
}
\end{wraptable}

\textbf{Implementation details.} To simulate the case where we only have a pretrained network on protected attributes, we train a Resnet-$18$ \cite{he2016deep} on {\it gender} attribute, and when we train the classifier to predict {\it attractiveness}, we add a loss using the canonical correlations between these two networks on intermediate layers: $L_{\text{total}} = L_{\text{cross-entropy}} + L_{\text{CCA}}$ where the first term is the standard cross entropy term and the second term is the canonical correlation. 
See appendix for more details of training/evaluation.

\textbf{Results.} We choose two commonly used error metrics for fairness: difference in Equality of Opportunity \cite{hardt2016equality} (DEO), and difference in Demographic Parity \cite{yao2017beyond} (DDP). We conduct experiments by applying the canonical correlation loss on three different layers in Resnet-18. In Table \ref{table_fair}, we can see that applying canonical correlation loss generally improves the DEO and DDP metrics (lower is better) over the standard model (trained using cross entropy loss only). Specifically, applying the loss on early layers like conv$0$ and conv$1$ gets better performance than applying at a relatively late layer like conv$2$. Another promising aspect of our approach is that is can easily handle the case where the protected attribute is a continuous variable (as long as a well trained regression network on the protected attribute is given) while other methods like \cite{lokhande2020fairalm,zhang2018mitigating} need to first discretize the variable and then enforce constraints which can be much more involved.

\textbf{Limitations.}
 Our current implementation has difficulty to scale beyond $d=10^5$ data dimension and this may be desirable for large scale DNNs. Exploring the sparsity may be one way to solve the problem and will be enabled by
 additional developments in modern toolboxes.

\section{Related Work}
{\bf Stochastic CCA:}
There has been much interest in designing scalable and provable algorithms for CCA: \cite{Ma2015FindingLS} proposed the first stochastic algorithm for CCA, while only local convergence is proven for non-stochastic version. \cite{wang2016efficient} designed algorithm which uses alternating SVRG combined with shift-and-invert pre-conditioning, with global convergence.  These stochastic methods, together with \cite{ge2016efficient} \cite{AllenZhu2016DoublyAM}, which reduce CCA problem to generalized eigenvalue problem and solve it by performing efficient power method, all belongs to the methods that try to solve empirical CCA problem, it can be seen as ERM approxiamtion of the original population objective, which requires solving numerical optimization of the empirical CCA objective on a fixed data set. These methods usually assume the access to the full dataset in the beginning, which is not very suitable for many practical applications where data tend to come in a streaming way. Recently, there are increasingly interest in considering population CCA problem \cite{arora2017stochastic} \cite{gao2019stochastic}. The main difficulty in population setting is we have limited knowledge about the objective unless we know the distribution of $\bold{X}$ and $\bold{Y}$. \cite{arora2017stochastic} handles this problem by deriving an estimation of gradient of population objecitve whose error can be properly bounded so that applying proximal gradient to a convex relexed objective will provably converge. \cite{gao2019stochastic} provides tightened analysis of the time complexity of the algorithm in \cite{wang2016efficient}, and provides sample complexity under certain distribution. The problem we are trying to solve in this work is the same as that in \cite{arora2017stochastic,gao2019stochastic}: to optimize the population objective of CCA in a streaming fashion.

{\bf Riemannian Optimization:}
Riemannian optimization is a generalization of standard Euclidean optimization methods to smooth manifolds, which takes the following form:
Given $f: \mathcal{M} \rightarrow \mathbf{R},$ solve $\min _{x \in \mathcal{M}} f(x)$, 
where $\mathcal{M}$ is a Riemannian manifold. 
  One advantage is that it provides a nice way to express many constrained optimization problems as unconstrained problems. Applications include matrix and tensor factorization \cite{Ishteva2011BestLM}, \cite{Tan2014RiemannianPF}, PCA \cite{Edelman1998TheGO}, CCA \cite{Yger2012AdaptiveCC}, and so on. \cite{Yger2012AdaptiveCC} rewrites CCA formulation as Riemannian optimization on Stiefel manifold. In our work, we further explore the ability of Riemannian optimization framework, decomposing the linear space spanned by canonical vectors into products of several matrices which lie in several different Riemannian manifolds. 

\section{Conclusions}

In this work, we presented
a stochastic approach (RSG+) for the CCA model
based on the observation that the solution of CCA can be decomposed into a product of matrices
which lie on certain structured spaces.
This affords specialized
numerical schemes and makes the optimization more efficient.
The optimization is based on Riemannian stochastic gradient descent and we provide a proof
for its $O(\frac{1}{t})$ convergence rate with $O(d^2k)$ time complexity per iteration.
In experimental evaluations, we find that our RSG+ behaves favorably relative to the baseline stochastic CCA method in
capturing the correlation in the datasets. We
also show the use of RSG+ in the DeepCCA setting showing feasibility when scaling to large dimensions as
well as in an interesting use case in training fair models. 


\nocite{langley00}

\bibliography{egbib}
\bibliographystyle{abbrvnat}

\section{Appendix}
\subsection{A brief review of relevant differential geometry concepts}
To make the paper self-contained, we briefly review certain differential geometry
concepts.
We only include a condensed description -- needed
for our algorithm and analysis -- and refer the interested reader to
\cite{boothby1986introduction} for a comprehensive and rigorous treatment of the topic. 

  \begin{figure}[!h]
   \setlength{\abovecaptionskip}{-0cm}
    \setlength{\belowcaptionskip}{-0cm} 
        \centering
               \includegraphics[width=0.7\columnwidth]{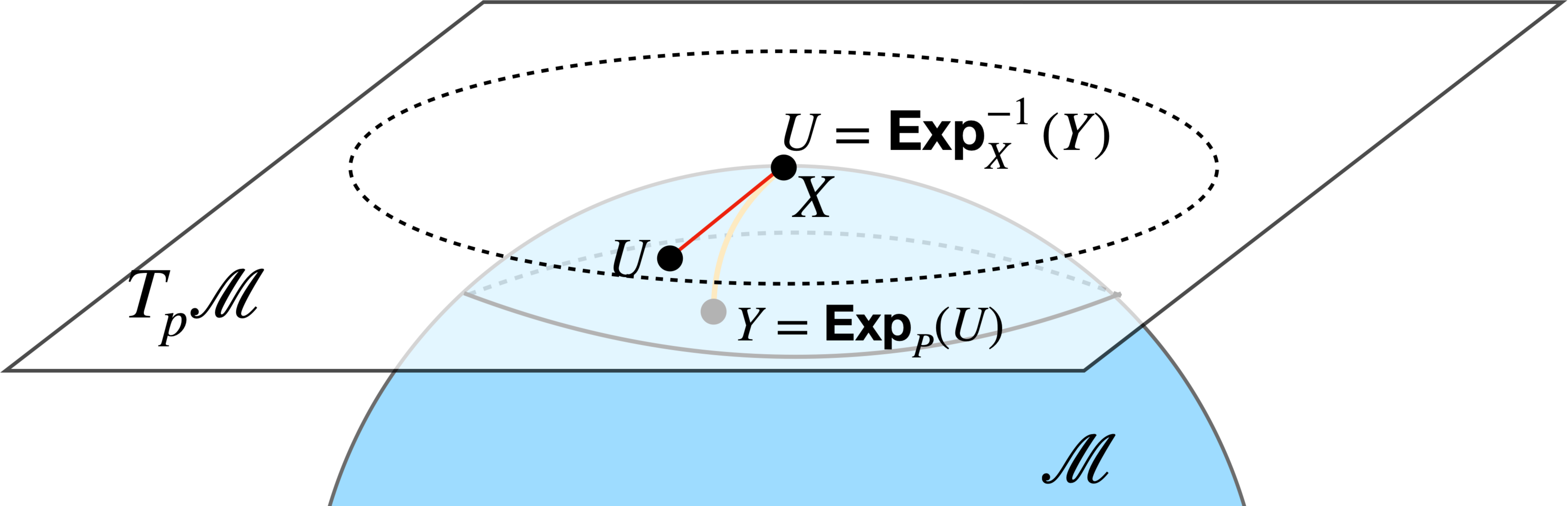}
               \caption{\label{tangent}\footnotesize Schematic description of an exemplar manifold ($\mathcal{M}$) and the visual illustration of $\textsf{Exp}$ and $\textsf{Exp}^{-1}$ map. }
\end{figure}

\textbf{Riemannian Manifold:} A Riemannian manifold, $\mathcal{M}$, (of dimension $m$) is defined as a (smooth) topological space which is locally diffeomorphic to the Euclidean space $\mathbf{R}^m$.
Additionally, $\mathcal{M}$ is equipped with a Riemannian metric $g$ which can be defined
as $$g_X: T_X\mathcal{M} \times T_X\mathcal{M}\rightarrow \mathbf{R},$$ where $T_X\mathcal{M}$
is the tangent space at $X$ of $\mathcal{M}$, see Fig. \ref{tangent}.

If $X\in \mathcal{M}$, the Riemannian Exponential map at $X$,
denoted by $\textsf{Exp}_X: T_X \mathcal{M} \rightarrow \mathcal{M}$ is defined as $\gamma(1)$
where $\gamma:[0,1]\rightarrow \mathcal{M}$. We can find $\gamma$ by solving the following differential equation: $$\gamma(0) = X, (\forall t_0 \in [0,1])\frac{d\gamma}{dt}\Big|_{t=t_0} = U.$$
In general $\textsf{Exp}_X$ is not invertible but the inverse $$\textsf{Exp}^{-1}_X:\mathcal{U}\subset \mathcal{M} \rightarrow T_X\mathcal{M}$$ is defined only if $\mathcal{U} = \mathcal{B}_r(X)$, where $r$ is called
the {\it injectivity radius} \cite{boothby1986introduction} of $\mathcal{M}$. This
concept will be useful to define the mechanics of gradient descent on the manifold.

In our reformulation, we made use of
the following manifolds, specifically, when decomposing $U$ and $V$ into a product of several matrices.
\begin{asparaenum}[(a)]
\item $\textsf{St}(p,n)$: the {\bf Stiefel} manifold consists of $n\times p$
  column orthonormal matrices
\item $\textsf{Gr}(p,n)$: the {\bf Grassman} manifold consists of
  $p$-dimensional subspaces in $\mathbf{R}^n$
\item $\textsf{SO}(n)$, the manifold/group consists of $n\times n$ {\bf special orthogonal matrices}, i.e., space of orthogonal matrices with determinant $1$.
\end{asparaenum}

{ 
{\bf Differential Geometry of $\textsf{SO}(n)$:} $\textsf{SO}(n)$ is a compact
Riemannian manifold, hence by the Hopf-Rinow theorem, it is also a
geodesically complete manifold \cite{helgason}. Its
geometry is well understood -- we recall a few relevant concepts here
and note that \cite{helgason} includes a more comprehensive treatment.

$\textsf{SO}(n)$ has a Lie group structure and the corresponding Lie
algebra, $\mathfrak{so}(n)$, is defined as, $$ \mathfrak{so}(n) = \{W
\in \mathbf{R}^{n\times n} | W^T = -W\}.$$  In other words,
$\mathfrak{so}(n)$ (the set of Left invariant vector fields with
associated Lie bracket) is the set of $n \times n$ anti-symmetric
matrices. The Lie bracket, $[,]$, operator on $\mathfrak{so}(n)$ is
defined as the commutator, i.e., 
$$\text{for } U, V \in \mathfrak{so}(n), \quad [U,
  V] = UV-VU.$$ Now, we can define a Riemannian metric on $\textsf{SO}(n)$ as
follows: $$ {\langle U, V \rangle}_{X} = \text{trace}\Big(U^TV\Big), \quad \text{where}$$ 
$$U, V \in T_X (\textsf{SO}(n)), X \in \textsf{SO}(n).$$ 

It can be shown that
this is a bi-invariant Riemannian metric. Under this bi-invariant
metric, now we define the Riemannian exponential and inverse
exponential map as follows. Let, $X, Y \in \textsf{SO}(n)$, $U \in T_{X}
(\textsf{SO}(n))$. Then, 
\begin{align*}
Exp_{X}^{-1}(Y) &= X \log(X^T Y) \end{align*}
\begin{align*}
Exp_{X} (U) &= X \exp(X^T U), 
\end{align*}
where, $\exp$, $\log$ are the matrix exponential and
logarithm respectively.

{\bf Differential Geometry of the Stiefel manifold:} The set of all
full column rank $(n \times p)$ dimensional real matrices form a
Stiefel manifold, $\textsf{St}(p, n)$, where $n \geq p$. 

A compact
Stiefel manifold is the set of all column orthonormal real
matrices. When $p<n$, $\textsf{St}(p, n)$ can be identified with
$$
\textsf{SO}(n)/SO(n-p).$$ Note that, when we consider the quotient space,
$\textsf{SO}(n)/SO(n-p)$, we assume that $\textsf{SO}(n-p) \simeq \iota(\textsf{SO}(n-p))$ is a
subgroup of $\textsf{SO}(n)$, where, 
$$\iota: \textsf{SO}(n-p) \rightarrow \textsf{SO}(n)$$ defined by
$$X\mapsto \begin{bmatrix} I_p & 0 \\ 0 & X\end{bmatrix}$$ is an
  isomorphism from $\textsf{SO}(n-p)$ to $\iota(\textsf{SO}(n-p))$.

{\bf Differential Geometry of the Grassmannian $\textsf{Gr}(p,n)$:} The
Grassmann manifold (or the Grassmannian) is defined as the set of all
$p$-dimensional linear subspaces in $\mathbf{R}^n$ and is denoted by
$\textsf{Gr}(p, n)$, where $p \in \mathbf{Z}^{+}$, $n \in
\mathbf{Z}^{+}$, $n \geq p$. Grassmannian is a symmetric space and can
be identified with the quotient space $$\textsf{SO}(n)/S\left(O(p)\times
O(n-p)\right),$$ where $S\left(O(p)\times O(n-p)\right)$ is the set of
all $n\times n$ matrices whose top left $p\times p$ and bottom right
$n-p \times n-p$ submatrices are orthogonal and all other entries are
$0$, and overall the determinant is $1$. 

A point $\mathcal{X} \in
\textsf{Gr}(p, n)$ can be specified by a basis, $X$.  We say that
$\mathcal{X} = \text{Col}(X)$ if $X$ is a basis of $\mathcal{X}$,
where $\text{Col}(.)$ is the column span operator. It is easy to see
that the general linear group $\text{GL}(p)$ acts isometrically,
freely and properly on $\textsf{St}(p, n)$. Moreover, $\textsf{Gr}(p, n)$
can be identified with the quotient space $\textsf{St}(p, n)/
\text{GL}(p)$. Hence, the projection map $$\Pi: \textsf{St}(p, n)
\rightarrow \textsf{Gr}(p, n)$$ is a {\it Riemannian submersion}, where
$\Pi(X) \triangleq \text{Col}(X)$. Moreover, the triplet
$(\textsf{St}(p,n), \Pi, \textsf{Gr}(p,n))$ is a fiber bundle.

{\it Horizontal and Vertical Space:} At every point $X \in \textsf{St}(p, n)$, we can define the {\it
  vertical space}, $\mathcal{V}_X \subset T_X \textsf{St}(p, n)$ to be
$\textrm{Ker}(\Pi_{*X})$. Further, given $g^{\textsf{St}}$, we define
the {\it horizontal space}, $\mathcal{H}_X$ to be the
$g^{\textsf{St}}$-orthogonal complement of $\mathcal{V}_X$.

{\it Horizontal lift:} Using 
the theory of principal bundles, for every vector field
$\widetilde{U}$ on $\textsf{Gr}(p, n)$, we define the {\it horizontal
  lift} of $\widetilde{U}$ to be the unique vector field $U$ on
$\textsf{St}(p, n)$ for which $U_X \in \mathcal{H}_X$ and $\Pi_{*X} U_X
= \widetilde{U}_{\Pi(X)}$, $\text{for all } X \in \textsf{St}(p, n)$. 

{\it Metric on \textsf{Gr}:} As, $\Pi$
is a Riemannian submersion, the isomorphism
$\Pi_{*X}|_{\mathcal{H}_X}: \mathcal{H}_X \rightarrow
T_{\Pi(X)}\textsf{Gr}(p, n)$ is an isometry from $(\mathcal{H}_X,
g^{\textsf{St}}_X)$ to $(T_{\Pi(X)}\textsf{Gr}(p, n),
g^{\textsf{Gr}}_{\Pi(X)})$. So, $g^{\textsf{Gr}}_{\Pi(X)}$ is defined as:
\begin{align}
\label{gr:metric}
g^{\textsf{Gr}}_{\Pi(X)}(\widetilde{U}_{\Pi(X)}, \widetilde{V}_{\Pi(X)}) &= g^{\textsf{St}}_{X} (U_X, V_X) \\
&= \textrm{trace}((X^TX)^{-1}U_X^TV_X) \nonumber
\end{align}
where, $\widetilde{U}, \widetilde{V} \in T_{\Pi(X)}\textsf{Gr}(p, n)$
and $\Pi_{*X} U_X = \widetilde{U}_{\Pi(X)}$, $\Pi_{*X} V_X =
\widetilde{V}_{\Pi(X)}$, $U_X \in \mathcal{H}_X$ and $V_X \in
\mathcal{H}_X$. 
}

We covered the exponential map and the Riemannian metric above, and their
explicit formulation for manifolds listed above is provided for easy reference
in Table \ref{manifoldExample}.

\begin{table*}[h!]
  \begin{center}
    \tabcolsep3pt 
    \begin{tabular}{l|c|c|c}
      \topline\myrowcolour
       & $g_X\left(U, V\right)$ & $\textsf{Exp}_X\left(U\right)$ & $\textsf{Exp}^{-1}_X\left(Y\right)$\\
      \midtopline
      $\textsf{St}(p,n)$ \cite{kaneko2012empirical}& $\text{trace}\left(U^TV\right)$ &  $\widetilde{U}\widetilde{V}^T$, & $(Y-X) - X(Y - X)^TX$\\
      & & $\widetilde{U}S\widetilde{V}^T = \text{svd}(X+U)$ & \\
      \midtopline
      $\textsf{Gr}(p,n)$ \cite{absil2004riemannian} & $\text{trace}\left(\Pi_*^{-1}\left(U\right)^T\Pi_*^{-1}\left(V\right)\right)$& $\widehat{U}\widehat{V}^T$, & $\bar{Y}\left(\bar{X}^T\bar{Y}\right)^{-1} - \bar{X}$,\\
      & & $\widehat{U}\widehat{S}\widehat{V}^T = \text{svd}(\bar{X}+U)$ & $X = \Pi(\bar{X}), Y = \Pi(\bar{Y})$\\
       \midtopline
      $\textsf{SO}(n)$ \cite{subbarao2009nonlinear}& $\text{trace}\left(X^TUX^TV\right)$ & $X\textsf{expm}\left(X^TU\right)$&  $X\textsf{logm}\left(X^TY\right)$\\
      \bottomline
    \end{tabular}
  \end{center}
  \caption{\footnotesize Explicit forms for some operations we need.
    $\Pi(X)$ returns $X$'s column space; $\Pi_*$ is $\Pi$'s differential.}
\label{manifoldExample}
\end{table*}

\subsection{Proof of Theorem 1}
We first restate the assumptions from section 2:	

\textbf{Assumptions:}
\begin{asparaenum}[\bfseries(a)]
\item The random variables $\bold{X} \sim \mathcal{N}(\mathbf{0}, \Sigma_x)$ and $\bold{Y} \sim \mathcal{N}(\mathbf{0}, \Sigma_y)$ with $\Sigma_x \preceq cI_d$ and $\Sigma_y \preceq cI_d$ for some $c>0$.
\item The samples $X$ and $Y$ drawn from $\mathcal{X}$ and $\mathcal{Y}$ respectively have zero mean.
\item For a given $k\leq d$, $\Sigma_x$ and $\Sigma_y$ have non-zero top-$k$ eigen values.
\end{asparaenum}

Recall that $F$ and $\widetilde{F}$ are the optimal values of the true and approximated CCA objective in (1) and (4) respectively,  we next restate Theorem 1 and give its proof:

\begin{theorem}
Under the assumptions and notations above, the approximation error $E = \|F-\widetilde{F}\|$ is bounded and goes to zero while the {\it whitening constraints} in (4b) are satisfied.
\end{theorem}


\begin{proof}
Let $U^*, V^*$ be the true solution of CCA. Let $U = \widetilde{U}S_uQ_u, V = \widetilde{V}S_vQ_v$ be the solution of (4) with $\widetilde{U}, \widetilde{V}$ be the PCA solutions of $X$ and $Y$ respectively with $S_uQ_u = \widetilde{U}^TU^*$ and $S_vQ_v = \widetilde{V}^TV^*$ (using RQ decomposition). Let $\widehat{X} = X\widetilde{U}\widetilde{U}^T$ and $\widehat{Y} = Y\widetilde{V}\widetilde{V}^T$ be the reconstruction of $X$ and $Y$ using principal vectors. 

Then, we can write 
\begin{align*}
\widetilde{F} &= \text{trace}\left(U^TC_{XY}V\right) 
= \text{trace}\left(\frac{1}{N} \left(\widehat{X}U^*\right)^T \widehat{Y}V^*\right)
\end{align*}
Similarly we can write $F = \text{trace}\left(\frac{1}{N}\left(XU^*\right)^TYV^*\right)$.

Using Def. 1, we know $\widehat{X}$, $\widehat{Y}$ follow sub-Gaussian distributions (such an assumption is common for such analyses for CCA as well as many other generic models).

Consider the approximation error between the objective functions as $E = |F - \widetilde{F}|$.  Due to von Neumann's trace inequality and Cauchy–Schwarz inequality, we have
\begin{align*}
\label{approx_eq}
E &= \frac{1}{N}|\text{trace}\left((U^*)^T\widehat{X}^T\widehat{Y}(V^*) - (U^*)^TX^TY(V^*)\right)|\\
& \leq |\text{trace}\left((U^*)^T\left( \left( \widehat{X} - X \right)^T\left(\widehat{Y} - Y\right) -2X^TY + X^T\widehat{Y} + \widehat{X}^TY\right)(V^*)\right)| \\& \leq 
 \sum\limits_{i} \sigma_i(\widehat{X}_u - X_u )\sigma_i(\widehat{Y}_v - Y_v) + \sum_i \sigma_i(\widehat{X}_u - X_u )\sigma_i(Y_v) + \sum_i \sigma_i(\widehat{Y}_v - Y_v )\sigma_i(X_u)
 \\& \leq \|\left( \widehat{X}_u - X_u \right)\|_F\|\left( \widehat{Y}_v - Y_v \right)\|_F + \|\left( \widehat{X}_u - X_u \right)\|_F\|Y_v\|_F + \left( \widehat{Y}_v - Y_v \right)\|_F\|X_u\|_F\quad \quad  \tag{A.1}
\end{align*}
Here $A_u = AU^*$ and $A_v = AV^*$ for any suitable $A$.
where $\sigma_i$(A) denote the $i$-th singular value of matrix A and $\|\bullet\|_F$ denotes the Frobenius norm. 

Now, using Proposition 1, we get  
\begin{align*}
\label{approx_ineq}
\|\left( \widehat{X}_u - X_u \right)\|_F &\leq \min\left(
\sqrt{2k}\|\Delta_x\|_2, \frac{2\|\Delta_x\|_2^2}{\lambda^x_k - \lambda^x_{k+1}}\right) \\ 
\|\left( \widehat{Y}_v - Y_v \right)\|_F &\leq \min\left(
\sqrt{2k}\|\Delta_y\|_2, \frac{2\|\Delta_y\|_2^2}{\lambda^y_k - \lambda^y_{k+1}}\right) \tag{A.2}
\end{align*}
where,
\begin{align}
\Delta_x = C(X_u) - C(\widehat{X}_u) \quad \Delta_y = C(Y_v) - C(\widehat{Y}_v).
\end{align}
Here $\lambda^x$s and $\lambda^y$s are the eigen values of $C(X_u)$ and $C(Y_v)$ respectively. Now, assume that $C(X_u) = I_k$ and $C(Y_v) = I_k$ since $X_u$ and $Y_v$ are solutions of Eq. 1. Furthermore assume $\lambda^x_k - \lambda^x_{k+1} \geq \Lambda$ and $\lambda^y_k - \lambda^y_{k+1} \geq \Lambda$ for some $\Lambda >0$. Then, we can rewrite \eqref{approx_eq} as

\begin{align*}
E &\leq \min\left(
\sqrt{2k}\|I_k - C(\widehat{X}_u)\|_2, \frac{2\|I_k - C(\widehat{X}_u)\|_2^2}{\Lambda}\right) \min\left(
\sqrt{2k}\|I_k - C(\widehat{Y}_v)\|_2, \frac{2\|I_k - C(\widehat{Y}_v)\|_2^2}{\Lambda}\right) + \\ &\min\left(
\sqrt{2k}\|I_k - C(\widehat{X}_u)\|_2, \frac{2\|I_k - C(\widehat{X}_u)\|_2^2}{\Lambda}\right) \|Y_v\|_F + \\ & \min\left(
\sqrt{2k}\|I_k - C(\widehat{Y}_v)\|_2, \frac{2\|I_k - C(\widehat{Y}_v)\|_2^2}{\Lambda}\right) \|X_u\|_F
\end{align*}

As $C(\widehat{X}_u) \rightarrow I_k$ or $C(\widehat{Y}_v) \rightarrow I_k$, $E\rightarrow 0$. Observe that the limiting conditions for $C(\widehat{X}_u)$ and $C(\widehat{Y}_v)$ can be satisfied by the ``whitening'' constraint. In other words, as $C(X_u) = I_k$ and $C(Y_v) = I_k$, $C(\widehat{X}_u)$ and $C(\widehat{Y}_v)$ converge to $C(X_u)$ and $C(Y_v)$, the approximation error goes to zero.
\end{proof}

\subsection{Implementation details of CCA on fixed dataset}
\label{appendix_implementation_fix_dataset}
\textbf{Implementation details.} On all three benchmark datasets, we only passed the data once for both our RSG+ and MSG \cite{arora2017stochastic} and we use the code from \cite{arora2017stochastic} to produce MSG results.
We conducted experiments on different dimensions of target space: $k=1,2,4$. The choice of $k$ is motivated by the fact that the spectrum of the datasets decays quickly. Since our RSG+ processes data in small blocks,
we let data come in mini-batches (mini-batch size was set to $100$). 





\subsection{Error metrics for fairness}
\label{appendix_fair_metric}
\textbf{Equality of Opportunity (EO) \cite{hardt2016equality}}: A classifier $h$ is said to satisfy EO if the prediction is independent of the protected attribute $s$ (in our experiment $s$ is a binary variable where $s=1$ stands for {\it Male} and $s=0$ stands for {\it Female}) for classification label $y\in\{0,1\}$. We use the difference of false negative rate (conditioned on $y=1$) across two groups identified by protected attribute $s$ as the error metric, and we denote it as DEO.

\textbf{Demographic Parity (DP) \cite{yao2017beyond}}: A classifier $h$ satisfies DP if the likelihodd of making a misclassification among the positive predictions of the classifier is independent of the protected attribute $s$. We denote the difference of demographic parity between two groups identified by the protected attribute as DDP.

\subsection{Implementation details of fairness experiments}
\label{appendix_implementation_fairness}
\textbf{Implementation details.} The network is trained for $20$ epochs with learning rate $0.01$ and batch size $256$. We follow \cite{donini2018empirical} to use NVP (novel validation procedure) to evaluate our result: first we search for hyperparameters that achieves the highest classification score and then report the performance of the model which gets minimum fairness error metrics with accuracy within the highest $90\%$ accuracies. When we apply our RSG+ on certain layers, we first use randomized projection to project the feature into $1$k dimension, and then extract top-$10$ canonical components for training. Similar to our previous experiments on DeepCCA, the batch method does not scale to $1$k dimension.

\textbf{Resnet-18 architecture and position of Conv-0,1,2 in Table 3.}
The Resnet-18 contains a first convolutional layer followed by normalization, nonlinear activation, and max pooling. Then it has four residual blocks, followed by average polling and a fully connected layer. We denote the position after the first convolutional layer as conv$0$, the position after the first residual block as conv$1$ and the position after the second residual block as conv$2$. We choose early layers since late layers close to the final fully connected layer will have feature that is more directly relevant to the classification variable ({\it attractiveness} in this case).

\begin{table}[!ht]
\caption{Results of \cite{Yger2012AdaptiveCC} (on CIFAR-10, our implementation of \cite{Yger2012AdaptiveCC} faces convergence issues).}
\centering
\smallskip
\begin{tabular}{l  c c c | c c c }
\hline
  &  \multicolumn{3}{c}{MNIST} & \multicolumn{3}{c}{Mediamill} \\
  Performance                     & $k=1$     & $k=2$    & $k=4$& $k=1$     & $k=2$    & $k=4$\\
\hline
PCC  & $0.93$ & $0.81$ & $0.53$  & $0.55$  & $0.61$  & $0.51$   \\
Time (s)   & $575.88$  & $536.46$  & $540.91$  & $41.89$  & $28.66$  & $28.76$    \\
\hline
\end{tabular}
\label{table_yger_2012}
\end{table}

\subsection{Comparison with \cite{Yger2012AdaptiveCC}}
We implemented the method from \cite{Yger2012AdaptiveCC} and conduct experiments on the three datasets above. The results are shown in Table \ref{table_yger_2012}. We tune the step size between $[0.0001,0.1]$ and $\beta=0.99$ as used in their paper. On MNIST and MEDIAMILL, the method performs comparably with ours except $k=4$ case on MNIST where it does not converge well. Since this algorithms also has a $d^3$ complexity, the runtime is $100\times$ more than ours on MNIST and 
$20\times$ more on Mediamill. On CIFAR10, we fail to find a suitable step size for convergence.


\end{document}